%% file: arvix-v2.tex
\documentclass[
sigconf = true, 
review = false, 
screen = true, 
anonymous = false, 
]{acmart}

\renewcommand\footnotetextcopyrightpermission[1]{} 

\usepackage{mathtools}
\usepackage{csquotes}
\usepackage{xspace}
\usepackage{epsfig}
\usepackage[ruled,linesnumbered]{algorithm2e}
\usepackage{color}
\usepackage{url}
\usepackage{enumitem}
\usepackage{graphicx}
\usepackage{tikz}
\usepackage{wasysym}
\usepackage{booktabs}
\usepackage{longtable}
\usepackage{array}
\usepackage{multirow}
\usepackage{wrapfig}
\usepackage{float}
\usepackage{colortbl}
\usepackage{pdflscape}
\usepackage{tabu}
\usepackage{threeparttable}
\usepackage{threeparttablex}
\usepackage[normalem]{ulem}
\usepackage{makecell}
\usepackage{xcolor}
\usepackage[shortcuts]{extdash}
\usepackage{balance}

\input{includes/commands.tex}
\input{includes/commenting_macros.tex}

\AtBeginDocument{%
  \providecommand\BibTeX{{%
    \normalfont B\kern-0.5em{\scshape i\kern-0.25em b}\kern-0.8em\TeX}}}

\copyrightyear{2023}
\acmYear{2023}
\setcopyright{acmlicensed}\acmConference[GECCO '23]{Genetic and Evolutionary Computation Conference}{July 15--19, 2023}{Lisbon, Portugal}
\acmBooktitle{Genetic and Evolutionary Computation Conference (GECCO '23), July 15--19, 2023, Lisbon, Portugal}
\acmPrice{15.00}
\acmDOI{10.1145/3583131.3590383}
\acmISBN{979-8-4007-0119-1/23/07}

\begin{document}


\title[Runtime Analysis of Quality Diversity Algorithms]{Runtime Analysis of Quality Diversity Algorithms}


\author{Jakob Bossek}
\orcid{0000-0002-4121-4668}
\affiliation{%
  \institution{
RWTH Aachen University}
  \city{Aachen}
  \country{Germany}
}

\author{Dirk Sudholt}
\orcid{0000-0001-6020-1646}
\affiliation{%
  \institution{
University of Passau}
  \city{Passau}
  \country{Germany}
}

\renewcommand{\shortauthors}{Bossek and Sudholt}

\begin{abstract}
Quality diversity~(QD) is a branch of evolutionary computation that gained increasing interest in recent years.
The Map-Elites QD approach defines a feature space, i.e., a partition of the search space, and stores the best solution for each cell of this space. We study a simple QD algorithm in the context of pseudo-Boolean optimisation on the ``number of ones'' feature space, where the $i$th cell stores the best solution amongst those with a number of ones in $[(i-1)k, ik-1]$. Here $k$ is a granularity parameter \changed{$1 \leq k \leq n+1$}. We give a tight bound on the expected time until all cells are covered for arbitrary fitness functions and for all $k$ and analyse the expected optimisation time of QD on \textsc{OneMax} and other problems whose structure aligns favourably with the feature space. On combinatorial problems we show that QD finds a ${(1-1/e)}$-approximation when maximising any monotone sub-modular function with a single uniform cardinality constraint efficiently. Defining the feature space as the number of connected components of a connected graph, we show that QD finds a minimum spanning tree in expected polynomial time.
\end{abstract}



\maketitle

\section{Introduction}
\label{sec:introduction}

Evolutionary algorithms~(EAs) are randomised search heuristics imitating principles of natural evolution of populations. EAs have proven to perform excellently in many domains, including -- but not limited to -- multi-objective optimisation, logistics, and transportation~\cite{deb_2012_optimisation}.
The classic focus of EAs is on pure (global) optimisation, i.e., to find a single high-performing solution to the problem at hand. In multi-modal optimisation~\cite{preuss_multimodal_2015} instead, the goal is in finding all or (at least) most of multiple local and/or global optima. Such algorithms usually use different means of explicit diversity preservation to allow a more exploratory optimisation process in order to prevent the population from converging into local optima which often act as traps.
Another related EA-branch first introduced by Ulrich~\&~Thiele~\cite{ulrich_maximizing_2011} in the continuous domain is evolutionary diversity optimisation~(EDO). In EDO the aim is to evolve a population of diverse solutions which are all high-performers (i.e., they all satisfy a minimum quality threshold which may be set a-priori or can be adapted dynamically). Usually, diversity is measured by a domain- or even problem-specific diversity function. Recently, EDO gained increasing interest in  combinatorial optimisation.
Very recently, a new branch of evolutionary computation termed quality diversity~(QD) emerged~\cite{chatzilygeroudis_quality-diversity_2021}. In the so-called \emph{MAP-Elites} approach~\cite{mouret_illuminating_2015}, the search space is partitioned into disjoint regions of interest termed \emph{cells}. The goal is to find high-quality solutions for each cell, making diversity explicit rather than implicit like it is done, e.g., in EDO. QD algorithms have major applications in the field of robotics~\cite{cully_evolving_2016} and only recently have been successfully applied  in, e.g., TSP instance space coverage~\cite{BosNeu22} or combinatorial optimisation for the knapsack problem~\cite{nikfarjam_use_2022} and the travelling thief problem~\cite{nikfarjam_analysis_2022}.

Stemming from engineering applications, the theory of evolutionary algorithms was mainly neglected in the early years. With some delay it finally developed into a mature field with many runtime results on pseudo-Boolean optimisation, many combinatorial optimisation problems~\cite{neumann2010bioinspired}, multi-modal optimisation~\cite{rajabi_stagnation_2021}, diversity preservation~\cite{friedrich_analysis_2009} and -- rather recently -- advances in multi-objective optimisation~\cite{zheng_first_2022} and EDO~\cite{do_analysis_2021}. However, at the moment of writing -- to the best of our knowledge -- there is just one work on the runtime analysis of quality diversity algorithms. Nikfarjam et al.~\cite{nikfarjam_analysis_2022} analyse a QD algorithm for the knapsack problem. They show that the studied QD algorithm operating on a suitable two-dimensional feature space mimics the dynamic programming approach for the knapsack problem; they show upper bounds on the expected optimisation time.
Apart from this, there is no theoretical work explaining the potential benefits of QD. Many questions need to be solved: when does QD perform well, and why? How does its performance compare with that of established algorithms?  How to design effective QD algorithms for interesting problem classes? Many of these questions can be resolved by runtime analysis.

We perform time complexity analysis of a simple MAP-Elites quality diversity algorithm, termed QD, in the context of pseudo-Boolean optimisation and combinatorial optimisation. QD operates on the $k$-number-of-ones~($k$-NoO) feature space where $k \in \{1, \ldots, n+1\}$ steers the granularity of the feature space, i.e., the number of cells. The $i$th cell of the map for $i = 1, \ldots, (n+1)/k=:L$ describes the niche of solutions with the number of ones in the interval $[(i-1)k,ik-1]$.

We consider three performance measures: the time to find an optimal search point (\emph{optimisation time}), the time to cover all cells (\emph{cover time}) and the time to find an optimal search point within each cell (\emph{optimal-cover time}).
We show that the cover time of QD operating on the $k$-NoO feature space on every pseudo-Boolean fitness function is $O(n^2 \log n)$ for $k=1$ and $O(n/(\sqrt{k}4^k p_m^k))$ for larger~$k$, where $p_m = \Theta(1/n)$ is the mutation probability. For granularity~$k=1$ this immediately gives upper bounds of $O(n^2 \log n)$ for all functions of unitation (where the fitness only depends on the number of ones) and all monotone functions~\cite{Doerr2012c,LenglerS18}, demonstrating that \QD can be powerful if the feature space is well aligned with relevant problem characteristics.
We show that the aforementioned bounds are tight for \ONEMAX and that the optimal-cover time on \ONEMAX is of the same order as well.

We further study QD on the $k$-NoO feature space in the domain of monotone submodular function maximisation which is known to be NP-hard in general. QD finds a $(1-1/e)$-approximation on any monotone submodular function with a single cardinality constraint $r$ in expected time $O(n^2(\log(n) + r))$, matching the upper bound shown by Friedrich~\&~Neumann~\cite{FN_2015_submodular_maximisation} for the multi-objective optimiser GSEMO. Both GSEMO and our QD operate similarly, but QD is simpler and more straight-forward from an algorithmic perspective. This application indicates an interesting relationship between the definition of a feature space in the context of quality diversity optimisation and the fitness-function construction -- encoding features of the encoded solutions as additional objectives -- in multi-objective optimisation.
Last but not least, we focus on the minimum spanning tree~(MST) problem on an edge-weighted, connected graph with $n$ nodes and $m$ edges. We define the feature space as the number of connected components~(CCs) of a connected graph encoded via a bit-string over $\{0,1\}^m$, i.e., the feature space spans a feature of the phenotype of the encoded solutions. QD operating on this CC feature space finds a minimum spanning tree of any connected edge-weighted source graph in at most $O(n^2m)$ steps, provided all edge weights are bounded by $2^{O(n)}$. Again, the algorithm works similarly to the bi-objective formulation studied by Neumann~\&~Wegener~\cite{Neumann2006c} who used the number of connected components and the total weight of the encoded solution as two conflicting objectives to be minimised simultaneously. They showed a bound of $O(nm(\log(n) + \log(w_{\max})))$ where $w_{\max}$ is the highest edge weight in the graph. QD is a more simple algorithm since it does not require the concept of Pareto-dominance to deal with partial ordering of the bi-objective fitness vectors, and the runtime analysis is more straight-forward as a consequence.


\section{The studied QD-algorithm}
\label{sec:algorithm}

\begin{algorithm}[ht]
\caption{QD algorithm}\label{alg:QD}
Initialise empty map $M$\;
Choose $x \in \{0,1\}^n$ uniformly at random\;
Store $x$ in map in cell $M(x)$\;
\While{termination condition not met}{
  Select parent $x \in M$ uniformly at random\;
  Generate $y$ from $x$ by bit-wise independent mutation with probability $p_m$\;
  \eIf{$M(y) \neq \emptyset{}$}{
    $M(y) = y$\;
  }{
    $z = M(y)$\;
    \If{$f(y) \geq f(z)$}{
      $M(y) = y$\;
    } 
  } 
} 
\end{algorithm}

The outline of the studied quality diversity algorithm \QD analysed is given in Algorithm~\ref{alg:QD}. The algorithm follows the \emph{Map-Elites} scheme and operates on functions $f \colon \{0,1\}^n \to \mathbb{R}$. First, QD initialises a \emph{map} $M$ which over the course of optimisation stores one solution for each expression of the feature space. The specific details will be discussed later; for now assume that $M$ is a function that maps $\{0,1\}^n$ to a feature space. In the beginning the map is empty. Next, an initial solution $x \in \{0,1\}^n$ is sampled uniformly at random. This solution is stored alongside its objective function value in the map in cell $M(x)$. After initialisation the evolutionary loop starts until a stopping condition is met.
In each iteration a random parent $x$ is sampled uniformly at random from the set of already covered map cells. A single offspring $y$ is generated from $x$ by bit-wise mutation with mutation probability $p_m$. If the cell at position $M(y)$ is not yet covered, $y$ is stored in any case. If $M(y)$ is already populated, say with $z$, the fitness function decides which one to keep. If $f(y)$ is no worse than $f(z)$, $y$ replaces $z$ in the respective map cell.

We define the following goals and corresponding hitting times.
\begin{itemize}
  \item \textbf{Optimisation}: let $T_{f, \text{OPT}}$ be the random number of function evaluations until QD finds a global optimum of~$f$.
  \item \textbf{Coverage}: let $T_{f, \text{C}}$ be the random number of function evaluations until QD finds each one solution for each map cell.
  \item \textbf{Optimal coverage:} Let $T_{f, \text{COPT}}$ be the time until an optimal solution is found for each map cell.
\end{itemize}
Note that the optimal-coverage time is the largest of all measures: $T_{f, \text{C}} \leq T_{f, \text{COPT}}$ and $T_{f, \text{OPT}} \leq T_{f, \text{COPT}}$.

\section{Feature space}
\label{sec:feature_space}

We will consider two feature spaces. The first is based on the number of ones. A second feature space for minimum spanning trees will be introduced later on, in Section~\ref{sec:runtime_mst}.
\begin{definition}
The \emph{number-of-ones} (NoO) feature space is defined as follows. Given a granularity parameter $k \in \{1, \dots, n+1\}$ with $k$ dividing $n+1$, the map $M$ is of size $L \coloneqq \frac{n+1}{k}$ and the $i$-th cell, $i \in \{1, \dots, L\}$, stores the best solution found so far amongst all solutions with a number of ones in $[(i-1)k, ik-1]$. We refer to this as the \noo{k} feature space.
%
\end{definition}
Note that $k=1$ matches the somehow natural feature space storing a best search point for every number of ones, that is, every level of the Boolean hypercube.
It should be noted that \QD can be implemented efficiently for the $k$-NoO feature space: the map can be represented via an array of length $L$. Given a solution $x \in \{0,1\}^n$, we can calculate the index $M(i)$ directly in constant time by calculating its array index $i = \lfloor|x|_1/k\rfloor$.

The number-of-ones feature space is a natural space for the Boolean hypercube and it is well-aligned with many well-studied fitness functions. Functions of unitation, e.g., \textsc{OneMax}, \textsc{Jump} and \textsc{Trap} only depend on the number of ones. These functions are defined in this way for simplicity and notational convenience. For some combinatorial optimisation problems, e.\,g.\ in selection problems like \textsc{Knapsack} or \textsc{Vertex Cover}, as well as for optimising submodular functions, the number of ones is a natural property.

\section{Pseudo-Boolean Optimisation}
\label{sec:runtime_test_functions}

\subsection{Upper Bounds for Cover Times}

We first give an upper bound on the expected time to cover all cells. This upper bound applies to every pseudo-Boolean fitness function and relies on worst-case probabilities of jumping from one cell to an undiscovered neighbouring cell.
\begin{theorem}
\label{thm:qdk_onemax_covtime_allones}
Consider \QD{} operating on the \noo{k} feature space with granularity \changed{$1 \leq k \leq n+1$}, $k$ dividing $n+1$, with an arbitrary initial search point and mutation rate $p_m = c/n$ for a constant~$c > 0$. For every fitness function, the expected cover time
is $O(n^2\log n)$ for $k=1$ and
$O\left(Lp_m^{-k}/\binom{2k-1}{k}\right) = O(n/(\sqrt{k}4^kp_m^k))$ for $k \geq 2$.
\end{theorem}

Note that, while Theorem~\ref{thm:qdk_onemax_covtime_allones} bounds the cover time irrespective of the fitness, the \emph{optimal-cover} time depends on the fitness function. In particular, the cell $\lceil L/2 \rceil$ contains all $\Omega(2^n/\sqrt{n})$ search points with $\lfloor n/2 \rfloor$ ones and the fitness function in this cell could be deceptive or a needle-in-a-haystack function, implying exponential optimal cover times. Hence the \emph{optimal}-cover time cannot be bounded as in Theorem~\ref{thm:qdk_onemax_covtime_allones}.

To prove Theorem~\ref{thm:qdk_onemax_covtime_allones}, we show the following technical lemma.
\begin{lemma}
\label{lem:binomial-coefficients-ki-minus-1}
For all $k \in \mathbb{N}$, $k \ge 2$ and $i \ge 3$,
\[
    \binom{ik-1}{k} \ge \frac{i^2}{3} \cdot \binom{2k-1}{k}.
\]
\end{lemma}
\begin{proof}
\begin{align*}
\frac{\binom{ik-1}{k}}{\binom{2k-1}{k}} =\;& \frac{(ik-1)!(k-1)!}{(ik-1-k)!(2k-1)!}\\
=\;& \frac{ik-1}{2k-1} \cdot \frac{ik-2}{2k-2} \cdot \ldots \cdot \frac{ik-k+1}{k+1} \cdot \frac{ik-k}{k}\\
\ge\;& \frac{ik-k+1}{k+1} \cdot \frac{ik-k}{k}
= \frac{ik-k+1}{k+1} \cdot (i-1).
\end{align*}
Since the fraction is non-decreasing with~$k$, its minimum is attained for $k=2$ and we obtain the lower bound
\begin{align*}
\frac{2i-2+1}{2+1} \cdot (i-1) = \frac{(2i-1)(i-1)}{3} = \frac{2i^2 - 3i + 1}{3}.
\end{align*}
Using $i \ge 3$, which implies $i^2 \ge 3i$,
we bound the numerator from below by $2i^2 - 3i + 1 \ge i^2 + 1 \ge i^2$. This completes the proof.
\end{proof}

\begin{proof}[Proof of Theorem~\ref{thm:qdk_onemax_covtime_allones}]
\added{We assume $k \le (n+1)/2$ as otherwise the only $k$ dividing $n+1$ is $k=n+1$, where we only have one cell and the cover time is trivial.}
After initialisation one cell $i$ is covered and $\QDk$ will maintain a solution $x^{(i)}$ whose number ones is in $[k(i-1), ki-1]$. Note that the precise search point can be replaced by another search point having the same property and a fitness value that is no worse.
If $i > 1$ then cell $i-1$ can be covered by choosing the current search point $x^{(i)}$ in cell~$i$ as a parent and creating an offspring whose number of ones is in $[k(i-2), k(i-1)]$. In the worst case\footnote{This worst case is likely to occur frequently on the function \ONEMAX, when trying to reach cells of smaller index, and on the function $-$\textsc{TwoMax} with $\textsc{TwoMax}(x) := \max\{\sum_{i=1}^n x_i, \sum_{i=1}^n (1-x_i)\}$.}, $\ones{x^{(i)}} = ki-1$ and then selecting $x^{(i)}$ as a parent and flipping exactly $k$ one-bits in $x^{(i)}$ will create an offspring in cell~$i-1$. Since the probability of selecting a parent from one out of at most $L$ covered cells is always at least $1/L$, the probability of this event is at least
\begin{align*}
    s_i
    & \geq \frac{1}{L} \cdot \binom{ki-1}{k} \cdot p_m^k (1-p_m)^{n-k}.
\end{align*}
The expected time for covering all cells with an index of at most~$i$ is bounded by $\sum_{j=2}^{i-1} \frac{1}{s_i} \le \sum_{j=2}^{L} \frac{1}{s_i}$. The same arguments apply symmetrically for reaching cells with a larger number of ones, assuming that in the worst case a search point in cell~$j$ has $k(j-1)$ ones and swapping the roles of zeros and ones.
Thus, the expected time to cover all cells is bounded by
\begin{align*}
    2\sum_{i=2}^{L} \frac{1}{s_i}
    & \leq \frac{2L}{p_m^{k}(1-p_m)^{n-k}} \cdot \sum_{i=2}^{L} \frac{1}{\binom{ki-1}{k}}.
\end{align*}
Now for the summation in the last line we make a case distinction.
For $k=1$ the summation simplifies to
\[
    \sum_{i=2}^L \frac{1}{\binom{i-1}{1}} = \sum_{i=2}^L \frac{1}{i-1} = \sum_{i=1}^{L-1} \frac{1}{i} = H(n),
\]
where $H(n)$ denotes the $n$-th harmonic number, as $k=1$ implies $L-1= \frac{n+1}{k} -1 = n$. Since $H(n) = \ln(n) + \Theta(1)$, along with $p_m = \Theta(1/n)$ and $(1-p_m)^{n-k} = (1-c/n)^{n-k} = \Theta(1)$, this proves the claimed upper bound for $k=1$.

For $k \geq 2$, applying Lemma~\ref{lem:binomial-coefficients-ki-minus-1} to all summands with \changed{$i \ge 3$},
\begin{align*}
    \sum_{i=2}^{L} \frac{1}{\binom{ki-1}{k}} \le\;& \frac{1}{\binom{2k-1}{k}} \cdot \left(1 +  3 \sum_{i=3}^L \frac{1}{i^2}\right) = \frac{O(1)}{\binom{2k-1}{k}}
\end{align*}
since $\sum_{i=3}^L \frac{1}{i^2} \le \sum_{i=1}^\infty \frac{1}{i^2} = \frac{\pi^2}{6} = O(1)$.
Hence, along with $(1-p_m)^{n-k} = (1-c/n)^{n-k} = \Theta(1)$, the expected coverage time is $O\left(Lp_m^{-k}/\binom{2k-1}{k}\right)$. Using well known results on the largest binomial coefficient (see, e.\,g.,~\cite[Equation (1.4.18)\added{, Corollary~1.4.12}]{DoerrProbabilityChapter2020}), $\binom{2k-1}{k} = 2^{2k-1}/\Theta(\sqrt{k}) = \Theta(4^k/\sqrt{k})$, we obtain the simplified bound $O(Lp_m^{-k} \cdot \sqrt{k}/4^k) = O(n/(\sqrt{k}4^kp_m^k))$.
\end{proof}

\subsection{Functions of Unitation \& Monotone Functions}

Our general Theorem~\ref{thm:qdk_onemax_covtime_allones} guarantees good performance for \QD if the feature space aligns favourably with the problem in hand. We give two such examples. Functions of unitation only depend on the number of ones in the bit string. Well-known examples include \ONEMAX$(x) \coloneqq \sum_{i=1}^n x_i$, \textsc{Jump}~\cite{Droste2002}, \textsc{Cliff}~\cite{JagerskupperStorch2007}, \textsc{Hurdle}~\cite{PRUGELBENNETT2004135} and \textsc{Trap}, that is, functions of wildly varying difficulty for evolutionary algorithms.
\changed{For functions of unitation and $k=1$ with the \noo{1} feature space, covering all cells is equivalent to covering all cells optimally, and to sampling a global optimum.} Hence Theorem~\ref{thm:qdk_onemax_covtime_allones} yields the following simple implication.
\begin{corollary}
\label{cor:unitation}
The expected cover time, the expected optimal-cover time, and the expected time until \QD{} operating on the \noo{1} feature space, with mutation rate $p_m = c/n$, $c > 0$ constant, finds a global optimum on any function of unitation is $O(n^2 \log n)$.
\end{corollary}
Note that for larger~$k$, the \emph{optimal}-cover time on functions of unitation can be larger than the cover time. If in cell~1 the all-zeros string is optimal and search points with $k-1$ ones have the second-best fitness of that cell, a mutation flipping $k-1$ specific bits is required to optimally cover this cell. Along with a factor of $\Theta(L)$ for selecting a parent from cell~1, this yields an expected time of $\Omega(L p_m^{-k+1})$ from a typical population covering cell~1 and $\Omega(L)$ other cells. If $k$ is large enough such that $4^k/\sqrt{k} = \omega(1/p_m)$, this is larger than the upper bound from Theorem~\ref{thm:qdk_onemax_covtime_allones}.

It should be pointed out that Corollary~\ref{cor:unitation} only holds since the feature space aligns with the definition of unitation functions. It is not robust in a sense that changing the encoding of unitation functions by swapping the meaning of zeros and ones for selected bits will immediately invalidate the proof.

A second example is the class of \emph{monotone} functions~\cite{Doerr2012c,LenglerS18,Lengler2020,LenglerZ21,KaufmannLLZ22}. A pseudo-Boolean function is called \emph{monotone} if flipping only 0-bits to~1 and not flipping any 1-bits strictly increases the fitness. This implies that $1^n$ is a unique global optimum. The class of monotone functions includes all linear functions with positive weights. All monotone functions can be solved by randomised local search in expected time $O(n \log n)$. But, surprisingly, there are monotone functions such that the (1+1)~EA with mutation rate $c/n$, $c > 0$ a sufficiently large constant, require exponential time with high probability~\cite{Doerr2012c,LenglerS18}. The reason is that mutation probabilities larger than $2.2/n$ frequently lead to mutations flipping both zeros and ones (thus avoiding the requirements of monotonicity) and accepting search points with a smaller number of ones. Since \QD{} on the \noo{1} feature space stores every increase in the number of ones, it runs in expected time $O(n^2 \log n)$ for \emph{every} mutation rate $c/n$.
\begin{corollary}
\label{cor:monotone-functions}
The expected time until \QD{} operating on the \noo{1} feature space, with mutation rate $p_m = c/n$, $c > 0$ constant, finds a global optimum on any monotone function is $O(n^2 \log n)$.
\end{corollary}
Even though every mutation that that flips only a 0-bit strictly increases the fitness, this property is not needed for Corollary~\ref{cor:monotone-functions} as empty cells are being filled regardless of fitness. But it suggests that the result on monotone functions may be more robust to fitness transformations than our result on functions of unitation.

\subsection{Tight Bounds for OneMax}

We also show that the upper bound from Theorem~\ref{thm:qdk_onemax_covtime_allones} is asymptotically tight for \ONEMAX, for all values of~$k$. For $k \ge 2$ the reason is that, on \ONEMAX, when cell~2 is reached, the best solution in cell~2 is one with $2k-1$ ones. The best chance of reaching cell~1 from there is to pick $M(2)$ as a parent and to flip $k$ out of $2k-1$ 1-bits.
This shows that even for the simplest possible fitness function (with a unique optimum), the expected cover time grows exponentially with the granularity~$k$ (except for the trivial setting of $k=n+1$, that is, having one cell only, for which \QD equals the (1+1)~EA and the expected optimal-cover time is $O(n \log n)$).
\begin{theorem}
\label{thm:lower-bound-cover-time-QDk}
Consider \QD operating on the \noo{k}  feature space, with $1 \leq k \leq (n+1)/2$ and mutation rate $p_m = c/n$, $c > 0$ constant, on \ONEMAX.
The expected cover time and the expected optimal-cover time are both in $\Theta(n^2\log n)$ for $k=1$ and
$\Theta\left(Lp_m^{-k}/\binom{2k-1}{k}\right) = \Theta(n/(\sqrt{k}4^kp_m^k))$ for $k \geq 2$.
\end{theorem}

We first show a technical lemma on decaying jump probabilities.
\begin{lemma}
\label{lem:transition-probabilities}
Let $p_{i, j}$ denote the probability of a standard bit mutation with mutation probability $p_m$ creating a mutant with $j$ ones from a parent with $i$ ones. For all $0 < p_m < 1$,
%
all $j < i$ and all $k \in \{0, \dots, j\}$,
\changed{
\[
    p_{i, j-k} \le \left(\frac{ip_m}{1-p_m}\right)^k \cdot p_{i, j}.
\]
}
\end{lemma}
\begin{proof}
For $j < i$ and $\ell \ge 0$ let $p_{i, j, \ell}$ denote the probability of a standard bit mutation with mutation probability~$p_m$ creating a mutant with $j$ ones from a parent with $i$ ones by flipping $i-j+\ell$ ones and flipping $\ell$ zeros, that is,
\[
    p_{i, j, \ell} = \binom{i}{i-j+\ell} \binom{n-i}{\ell} p_m^{i-j+2\ell} (1-p_m)^{n-i+j-2\ell}.
\]
\changed{Note that $p_{i, j, \ell} = 0$ if $\ell > \min\{j, n-i\}$.}
We claim that for all $i \in \{2, \dots, n\}$, all $d \in \{1, \dots, i\}$ and all $\ell \in \mathbb{N}_0$
\changed{
\begin{equation}
\label{eq:bound-p-i-less-d-ell}
    p_{i, i - d, \ell} \ge \frac{1-p_m}{ip_m} \cdot p_{i, i - d - 1, \ell}.
\end{equation}
}
This is trivial for $p_{i, i - d - 1, \ell} = 0$, hence we assume $p_{i, i - d - 1, \ell} > 0$ (which implies $i - d - \ell \ge 1$) and have
\changed{
\begin{align*}
     \frac{p_{i, i-d, \ell}}{p_{i, i-d-1, \ell}} =\;& \frac{\binom{i}{d+\ell}\binom{n-i}{\ell} p_m^{d+2\ell}(1-p_m)^{n-d-2\ell}}{\binom{i}{d+\ell+1} \binom{n-i}{\ell} p_m^{d+2\ell+1}(1-p_m)^{n-d-2\ell-1}}\\
     =\;& \frac{d+\ell+1}{i-d-\ell} \cdot \frac{1-p_m}{p_m}.
\end{align*}
}
Since $d + \ell \ge 1$ and $i \ge 2$, we have $\frac{d+\ell+1}{i-d-\ell} \ge \frac{2}{i-1} \ge \frac{1}{i}$ and this implies~\eqref{eq:bound-p-i-less-d-ell}. Consequently,
\changed{
\begin{equation*}
    p_{i, i-d} = \!\!\!\!\!\!\!\!\!\!\! \sum_{\ell=0}^{\min\{i-d,n-i\}} \!\!\!\!\!\!\!\!\!\! p_{i, i - d, \ell} \ge \!\!\!\!\!\!\!\!\!\! \sum_{\ell=0}^{\min\{i-d,n-i\}} \!\! \frac{1-p_m}{ip_m} \cdot p_{i, i - d - 1, \ell} =  \frac{1-p_m}{ip_m} \cdot  p_{i, i-d-1}.
\end{equation*}
}
The statement now follows from applying~\eqref{eq:bound-p-i-less-d-ell} repeatedly and bounding the resulting factors as
\changed{
\[
    \frac{1-p_m}{ip_m} \cdot \frac{1-p_m}{(i-1)p_m} \cdot \ldots \cdot \frac{1-p_m}{(i-k+1)p_m} \ge \left(\frac{1-p_m}{ip_m}\right)^k. \qedhere
\]}
%
\end{proof}

The following lemma shows that, with probability $\Omega(1)$, the map will be partially filled by the time the cells with lower indices will be reached. This implies that the probability of selecting a parent from the cell with smallest index is decreased by a factor of $\Theta(1/L)$, slowing down the exploration of smaller cells.
\begin{lemma}
\label{lem:covering-Omega-L-cells}
Consider the scenario of Theorem~\ref{thm:lower-bound-cover-time-QDk}.
There exists a constant $\varepsilon > 0$ such that, with probability $\Omega(1)$, \QD reaches a state where at least $\varepsilon L$ cells are covered and cells~1 to $\varepsilon L$ are not yet covered.

\end{lemma}
\begin{proof}
If $k=(n+1)/2$ then since $k$ divides $n+1$ we know that $n+1$ is even and there are just two symmetric cells $1$ and $2$. By symmetry, the probability of initialising in cell~2 is $1/2$ and then the claim follows.
The same arguments apply to values $k < (n+1)/2$ and $k = \Omega(n)$ as then $L = O(1)$ and covering one cell with a solution of at least $n/2$ ones is sufficient when choosing $\varepsilon := 1/L$.


\changed{
In the following, we assume $k = o(n)$ and define $X_t$ as the smallest number of ones in any search point seen so far. Note that, trivially, $X_{t+1} \le X_t$.
We claim that with probability $1-o(1)$ the following statements hold for a function $\alpha \coloneqq \alpha(n) \coloneqq \min\{n/3, n/(4c)\}$. (1) the algorithm is initialised with $X_0 \ge \alpha$, (2) whenever $X_t$ is decreased, it is decreased by at most $\alpha/4$, (3) while $X_t \in [1, \alpha]$, steps decreasing $X_t$ have an exponential decay (as will be made precise in the following). Claim (1) follows from Chernoff bounds since $\alpha \le n/3$. Claim (2) follows since decreasing $X_t$ by at least $\alpha/4$ requires flipping $\Omega(n)$ bits. This has probability at most $\binom{n}{\alpha/4} p_m^{\alpha/4} \le 2^n \cdot (c/n)^{\alpha/4} = n^{-\Omega(n)}$ whereas accepted steps decreasing $X_t$ have probability at least $\Omega(n^{-k}/L) = n^{-o(n)}$. Along with a union bound over $O(n)$ values of $X_t$, the conditional probability of decreasing $X_t$ by at least $\alpha/4$ throughout the run is $O(n) \cdot n^{-\Omega(n)}/n^{-o(n)} = n^{-\Omega(n)}$.
It remains to prove (3).
}


\changed{
Assuming $X_t \le \alpha$, we consider the decrease of $X_t$ in steps in which $X_{t+1} < X_t$ and consider the difference $D_t := X_{t+1} - X_t$.}
Adopting the notation from Lemma~\ref{lem:transition-probabilities}, we have
\changed{
\begin{align*}
    \prob{D_t \ge \ell \mid X_{t+1} < X_t, X_t = i}
    =\;& \frac{\sum_{j=0}^{i-\ell} p_{i, j}}{\sum_{j=0}^{i-1} p_{i, j}}
    \le \frac{\sum_{j=0}^{i-\ell} p_{i, j}}{\sum_{j=\ell-1}^{i-1} p_{i, j}}.
\end{align*}
Recalling $\alpha \le n/(4c)$, by Lemma~\ref{lem:transition-probabilities} we have for all $i \le \alpha$ and all $\ell \in \{0, \dots, i-1\}$.
\[
    p_{i, j} \ge \left(\frac{1-p_m}{ip_m}\right)^{\ell-1} \cdot p_{i, j-\ell+1} \ge \left(\frac{n}{2ci}\right)^{\ell-1} \cdot p_{i, j-\ell+1} \ge 2^{\ell-1}  p_{i, j-\ell+1}
\]
for $n$ large enough such that $1-p_m \ge 1/2$.} Hence the above is at most
\begin{align*}
    \;& \frac{\sum_{j=0}^{i-\ell} p_{i, j}}{2^{\ell-1} \sum_{j=\ell-1}^{i-1} p_{i, j-\ell+1}}
    = \frac{\sum_{j=0}^{i-\ell} p_{i, j}}{2^{\ell-1} \sum_{j=0}^{i-\ell} p_{i, j}}
    = 2^{-(\ell-1)}.
\end{align*}
Thus, $(D_t \mid X_{t+1} < X_t)$ is stochastically dominated by a geometric random variable $Z_t$ with parameter $1/2$.
In other words, the progress towards $0^n$ in steps in which $X_t$ decreases is thus stochastically dominated by a sequence $Z_0, Z_1, \dots $ of geometric random variables with parameter $1/2$.
\changed{Now choose $\varepsilon \coloneqq \alpha/(12Lk)$ and note $\varepsilon = \Theta(1)$.}

\changed{
Consider a phase of $\alpha/12$ steps. The expected total decrease in $\alpha/12$ steps is at most $\alpha/6$.
By Chernoff bounds for sums of geometric random variables (Theorem~10.32 in~\cite{DoerrProbabilityChapter2020}), the probability of the total decrease being at least $\alpha/4$ is at most $e^{-\Omega(\alpha)} = e^{-\Omega(n)}$.
Hence, with the converse probability $1 - e^{-\Omega(n)}$ at least $\alpha/12$ steps decreasing $X_t$ are necessary to reduce $X_t$ by a total of $\alpha/4$. By claims (1) and (2), with high probability the first $X_t$-value of at most $\alpha$ is at least $3\alpha/4$, hence the final $X_t$-value is at least $\alpha/2$.
}

\changed{
Since $\alpha/2 = 6\varepsilon L k$, the event $X_t \ge \alpha/2$ implies that the first $\varepsilon L$ cells have not been reached yet.
Since at most $k$ decreasing steps may concern the same cell, during $\alpha/12$ steps decreasing $X_t$, at least $\alpha/(12k) = \varepsilon L$ cells are visited.
%
Observing that the union bound of all failure probabilities is $o(1)$ completes the proof.
}
%
\end{proof}

Now we put everything together to prove Theorem~\ref{thm:lower-bound-cover-time-QDk}.
\begin{proof}[Proof of Theorem~\ref{thm:lower-bound-cover-time-QDk}]
The upper bounds on the expected cover time follow from Theorem~\ref{thm:qdk_onemax_covtime_allones}. For $k=1$ the cover time equals the optimal-cover time. For $k \ge 2$ we use a simple and crude argument: as long as not all cells are covered optimally, there is a cell~$i$ whose search point can be improved. An improvement happens with probability $1/L \cdot 1/(en)$, if the right parent is selected and an improving Hamming neighbour is found. The expected waiting time is at most $enL$ and at most $n$ such steps are sufficient to cover all cells optimally. This only adds a term $O(n^2L) = O(n^3)$ to the upper bound, which is no larger than the claimed runtime bound.

It remains to show the claimed lower bounds for the expected cover time.
By Lemma~\ref{lem:covering-Omega-L-cells}, with probability $\Omega(1)$, \QDk reaches a state where at least $\varepsilon L$ cells are covered and the first $\varepsilon L$ cells are not yet covered. We work under the condition that this happens and note that this only incurs a constant factor in the lower bound.

We first show the statement for $k \ge 2$ and $k \le \varepsilon(n+1)/2$. \added{We may assume that $\varepsilon (n+1)/2 \le n/(16c)$ as otherwise we may simply choose a smaller value for $\varepsilon$. This ensures that the exponential decay of jump lengths shown in the proof of Lemma~\ref{lem:covering-Omega-L-cells} applies for all search points with at most $4k \le 4\varepsilon(n+1)/2$ ones.}
Let $S$ be the set of all search points with at most $2k-1$ ones. As the first $\varepsilon L$ cells are not covered yet, the minimum number of ones seen so far is at least $\varepsilon k L = \varepsilon(n+1) \ge 2k$.
Applying Lemma~\ref{lem:transition-probabilities} as in the proof of Lemma~\ref{lem:covering-Omega-L-cells} we get that, for all \changed{$2k \le i \le 4\varepsilon(n+1)/2$}, the probability that a search point with exactly $2k-1$ ones is reached, when $S$ is reached for the first time, is
\begin{align*}
    & \Prob{X_{t+1} = 2k-1 \mid X_t = i \wedge X_{t+1} \le 2k-1} = \frac{p_{i, 2k-1}}{\sum_{j=0}^{2k-1} p_{i, j}}\\
    \ge\;& \frac{p_{i, 2k-1}}{\sum_{j=0}^{2k-1} 2^{-(2k-1-j)} p_{i, 2k-1}} = \frac{1}{\sum_{j=0}^{2k-1} 2^{-j}} \ge \frac{1}{2}.
\end{align*}
\added{The same probability bound also holds for $i > 4\varepsilon(n+1)/2$ as the conditional probability of jumping to $S$ from a search point with more than $4\varepsilon(n+1)/2$ ones is asymptotically smaller than the probability of jumping to $S$ from cell~3.}
Since we are optimising \ONEMAX, the algorithm will maintain a search point with exactly $2k-1$ ones in cell~2 forever. To discover cell~1, the algorithm either needs to pick the search point from cell~2 as parent and jump to a search point of at most $k-1$ ones, or pick a parent from some cell~$i$, $i \ge 3$, that has at least $(i-1)k$ ones and jump to the same target.
The probability of any such event is at most
\begin{align*}
& \frac{1}{\varepsilon L} \left(\sum_{j=0}^{k-1} p_{2k-1, j} + \sum_{i=3}^{L}\sum_{j=0}^{k-1} p_{(i-1)k, j}\right)\\
\intertext{\changed{For all $m \ge k-1$, since at least $m-(k-1)$ bits have to flip, $\sum_{j=0}^{k-1} p_{m, j} \le \binom{m}{m-(k-1)} p_m^{m-k+1} = \binom{m}{k-1} p_m^{m-k+1}$, thus this is at most}}
& \le\;\frac{1}{\varepsilon L} \left(\binom{2k-1}{k-1}p_m^k + \sum_{i=3}^{L} \binom{(i-1)k}{k-1} p_m^{(i-2)k+1}\right)
\end{align*}
Using that for $k \ge 2$
\[
    \binom{(i-1)k}{k-1}
    \!\!\ \le \!\!
    \frac{((i-1)k)^{k-1}}{(k-1)!}
    \!\! \le \!\! (i-1)^{k-1} \frac{k^{k-1}}{(k-1)!} \le (i-1)^{k-1} \binom{2k-1}{k-1}
\]
we get, for $i \ge 4$,
\begin{align*}
    \binom{(i-1)k}{k-1} p_m^{(i-2)k+1} \le\;& (i-1)^{k-1} \binom{2k-1}{k-1} p_m^{(i-2)k+1}\\
    =\;& \binom{2k-1}{k-1} p_m^k \cdot \frac{p_m}{i-1} \left((i-1)p_m^{(i-3)}\right)^k\\
    \le\;& \binom{2k-1}{k-1} p_m^k \cdot \frac{p_m}{i-1}
\end{align*}
as $(i-1)p_m^{i-3} \le 1$ for all $i \ge 4$ and large enough~$n$.

The summand for $i=3$ is
\[
    \binom{2k}{k-1} p_m^{k+1} = \binom{2k-1}{k-1} \cdot \frac{2k}{k+1} p_m^{k+1} \le \binom{2k-1}{k-1} 2p_m^{k+1}.
\]
Together, the sought probability is bounded by
\begin{align*}
     & \frac{1}{\varepsilon L} \left(\binom{2k-1}{k-1}p_m^k + 2p_m \sum_{i=3}^{L} \binom{2k-1}{k-1} p_m^{k}\right)\\
     =\;&  \frac{1}{\varepsilon L} \left(\binom{2k-1}{k-1}p_m^k \left(1 + 2p_m L\right)\right)
     = O\left(\binom{2k-1}{k-1}\frac{p_m^{k}}{L}\right).
\end{align*}
Taking the reciprocal yields a lower bound on the expected time for finding cell~1.

For $\varepsilon(n+1)/2 < k \le (n+1)/2$ all cells have linear width. Since the probability of initialising with at least $n/2 + \sqrt{n}$ ones is $\Omega(1)$ by properties of the binomial distribution, it is easy to show that the algorithm will reach a search point with $2k-1$ ones before reaching cell~1, with high probability. Then we apply the above arguments.


For $k=1$ we argue differently as the time is no longer dominated by the expected waiting time to jump from cell~2 to cell~1.
We assume that the event stated in Lemma~\ref{lem:covering-Omega-L-cells} occurs and consider $X_t$ as the minimum number of ones in any search point of the population at time $t$, $P_t$.
Let $Q_t$ be the number of ones in the parent chosen in generation~$t$. Define the drift
\[
    \Delta_i \coloneqq \E{X_{t}-X_{t+1} \mid X_t = s, Q_{t+1} = X_t + i}
\]
under the condition that a parent is chosen with a number of ones by $i$ larger than $X_t$. Note that $\Delta_0$ is bounded by the expected number of flipping 1-bits,
$
    \Delta_0 \le s \cdot p_m = \frac{cs}{n}$.
We focus on $X_t \le n^{1/5}$ and show that here the drift is  at most
\begin{equation}
\label{eq:drift}
    \hspace*{-0.2cm} \E{X_t \! - \! X_{t+1} \mid X_t = s} \! = \! \frac{1}{|P_t|} \! \sum_{x \in P_t} \! \Delta_{\ones{x}-s}
    \le \frac{1}{\varepsilon L} \sum_{i=0}^{n-s} \Delta_i = O(s/n^2)\!\!\!
\end{equation}
where the first equation follows from the law of total probability and the last equality remains to be proven. We will show that for all $i \ge 1$, $\Delta_i$ is asymptotically much smaller than $\Delta_0$.
We trivially have $\Delta_i \le \Prob{X_{t+1} < X_t \mid X_t, Q_t = X_t + i} \cdot (X_t+i)$ since the maximum possible decrease is $X_t+i$.
In order to decrease $X_t$ by mutating a parent of $X_t+i$ ones, it is necessary to flip at least $i+1$ ones. For $i \le X_t + n^{1/5}$ This event has probability
\begin{align*}
    \binom{X_t+i}{i+1} p_m^{i+1} \le \binom{2n^{1/5}}{i+1} p_m^{i+1}
    \le (2n^{1/5}p_m)^{i+1} \le (2cn^{-4/5})^{i+1}
\end{align*}
and thus $\Delta_i \le (2cn^{-4/5})^{i+1} \cdot 2n^{1/5} \le O(n^{-7/5})$ and $\sum_{i=0}^{n^{1/5}-1} \Delta_i = O(n^{-6/5})$. For all $i \ge X_t + n^{1/5}$ the above probability is at most
\begin{align*}
    \binom{X_t+i}{i+1} p_m^{i+1} \le \frac{(X_t + i)^{i+1}}{i!}  p_m^{i+1}
    \le \frac{n^{i+1}}{i!} p_m^{i+1} = \frac{c^{i+1}}{i!}.
\end{align*}
Since $i! \ge (i/e)^i$, this term is $n^{-\Omega(n^{1/5})}$ and $\sum_{i=n^{1/5}}^{n-X_t} \Delta_i = n^{-\Omega(n^{1/5})}$.
Along with $L = n+1$, \eqref{eq:drift} follows.

Finally, we apply the multiplicative drift theorem for lower bounds, Theorem~2.2 in~\cite{Witt2013}, to the process $X_t$, using the following parameters: $s_{\min} \coloneqq \log^2 n$ and $\beta \coloneqq 1/\ln n$. The condition on the drift,
$\E{X_t - X_{t+1} \mid X_t = s} \le \delta s$, is satisfied for $\delta \coloneqq c'/n^2$, $c'$ being a sufficiently small positive constant.

We still need to show that, for all $s \ge s_{\min} = \log^2 n$,
\[
    \Prob{X_t - X_{t+1} \ge \beta s \mid X_t = s} \le \frac{\beta \delta}{\ln n} = \frac{c'}{n^2(\ln n)^2}.
\]
Since $\beta s \ge \beta s_{\min} = \Omega(\log n)$, the above bound holds since the probability of flipping a logarithmic number of bits with $p_m = c/n$ is superpolynomially small.

By the multiplicative drift theorem, the expected time for reaching a distance of at most $s_{\min}$, and thus the time to reach cell~1, starting from a distance of at least $\varepsilon L = \varepsilon n$, is at least
\begin{align*}
    & \frac{\ln(\varepsilon n) - \ln(s_{\min})}{\delta} \cdot \frac{1-\beta}{1+\beta}\\
    =\;& \frac{n^2(\ln(\varepsilon n) - 2\ln(\log n))}{c'} \cdot \frac{1-1/(\ln n)}{1+1/(\ln n)}
    = \Omega(n^2 \log n). \qedhere
\end{align*}
\end{proof}

As an aside, we point out that Theorem~\ref{thm:lower-bound-cover-time-QDk} gives a lower bound on the expected time for GSEMO to cover the whole Pareto front of the bi-objective test function \textsc{OneMinMax}$(x) \coloneqq (\ONEMAX(x), n - \ONEMAX(x))$. To our knowledge, this is a novel (albeit not surprising) result; previously only for SEMO a lower bound of $\Omega(n^2 \log n)$ was known~\cite{Covantes2018b}. SEMO only flips one bit in each mutation, and previous work avoided the complexity of analysing standard mutations.
\begin{theorem}
The expected time for \QD, operating on the \noo{1} feature space, on \ONEMAX to cover all $n+1$ cells equals the expected time of GSEMO covering the Pareto front of \textsc{OneMinMax}.
\end{theorem}
\begin{proof}
Both algorithms keep one search point in every cell and create new search points by choosing a covered cell uniformly at random, applying standard bit mutation and keeping the result if it covers an uncovered cell.
\end{proof}

\section{Monotone submodular function optimisation}
\label{sec:runtime_submodular_optimisation}


We now turn towards the NP-hard problem of submodular function maximisation. Submodular functions generalise many well-known NP-hard combinatorial optimisation problems, e.g., finding a maximum cut in a graph, and are thus of utmost relevance. Let $\Omega$ be a finite ground set and let $f : 2^{\Omega} \to \mathbb{R}$ be a set function. The function $f$ is \emph{submodular} if for all $A, B \subseteq \Omega$
\begin{align*}
    f(A \cup B) + f(A \cap B) \leq f(A) + f(B)
\end{align*}
and $f$ is \emph{monotone} if $f(A) \leq f(B)$ for all $A \subseteq B \subseteq \Omega$. \Wlog we assume that $f$ is normalised, i.e., $f(A) \geq 0 \,\forall A \subseteq \Omega$ and ${f(\emptyset) = 0}$. Given a constraint $r$ the goal is to find a set $\OPT \subseteq \Omega$ with $f(\OPT) = \max_{A \subseteq \Omega, |A|\leq r} f(A)$.
Friedrich~\&~Neumann~\cite{FN_2014_submodular_maximisation,FN_2015_submodular_maximisation} were the first to study the performance of evolutionary algorithms on monotone submodular functions. In particular they showed that GSEMO, using a bi-objective fitness function, obtains a $(1-1/e)$-approximation on any monotone submodular function with a single uniform matroid constraint in expected time $O(n^2 (\log(n) + r))$.
\begin{algorithm}[t]
\caption{Global SEMO algorithm}\label{alg:gsemo}
Choose $x \in \{0,1\}^n$ uniformly at random\;
$P = \{x\}$\;
\While{termination condition not met}{
  Select $x \in P$ uniformly at random\;
  Generate $y$ from $x$ by bit-wise independent mutation with probability $1/n$\;
  \If{$y$ is not strictly dominated by any other $x' \in P$}{
    $P = P \cup \{y\}$\;
    Delete all solutions $z \in P$ that are dominated by $y$\;
  } 
} 
\end{algorithm}
The authors consider the fitness function $g(x) := (z(x), |x|_0)$ where $z(x) = f(x)$ if the solution $x$ is feasible (i.e., $|x|_1 \leq r$) and $z(x) = -1$ otherwise. They use the concept of \emph{Pareto-dominance} and define $g(x) \geq g(y)$ if and only if $((z(x) \geq z(y)) \land (|x|_0 \geq |y|_0))$ holds; in this case $x$ (Pareto-)dominates $y$. They study the expected runtime of the GSEMO (see Algorithm~\ref{alg:gsemo}) until a solution $x^{*} = \argmax_{x \in P} z(x)$ with $f(x^{*})/\OPT \geq \alpha$ is hit for the first time.

In the following we show that \QD with $p_m \coloneqq 1/n$ also achieves an approximation ratio of $(1-1/e)$ in polynomial expected time $O(n^2(\log(n) + r))$ using the function $f$ itself as the fitness function.

\begin{theorem}
\label{thm:qd1_submod_uniform}
The expected time until \QD operating on the \noo{1} feature space has found a $(1-1/e)$-approximate solution for a monotone submodular function with a uniform cardinality constraint $r$ is $O(n^2 (\log(n) + r))$.
\end{theorem}

The proof follows the proof of Theorem~2 in~\cite{FN_2015_submodular_maximisation} and the proof of the approximation quality of the deterministic greedy algorithm by Nemhauser et al.~\cite{NWF_1978_approx_submodular_maximisation}.
This is because the working principles of \QD and GSEMO match.
While the fitness function for GSEMO cleverly encodes the cardinality constraint as a second objective and uses Pareto-dominance to decide which solutions to keep, the \noo{1} feature space does the job for QD.

\begin{proof}[Proof of Theorem~\ref{thm:qd1_submod_uniform}]
By Theorem~\ref{thm:qdk_onemax_covtime_allones} the first cell of the map is populated with the solution $0^n$ in $O(n^2 \log n)$ steps in expectation. Since $0^n$ is the only solution with $n$ zeroes it will not be replaced.

More precisely, we show that once $0^n$ is stored in the map, \QD evolves solutions $x_j$ for $0 \leq j \leq r$ with approximation ratio
\begin{align}\label{eq:submod_partial_approx}
    f(x_j) \geq \left(1 - \left(1 - \frac{1}{r}\right)^j\right) \cdot f(\OPT)
\end{align}
in expected time $O(n^2r)$. The proof is by induction over $j$. As a base case we note that once the all zeroes string is in the map, Eq.~\eqref{eq:submod_partial_approx} holds for $x_0 = 0^n$ since
\begin{align*}
    0 = f(x_0) \geq \left(1 - \left(1 - \frac{1}{r}\right)^0\right) \cdot f(\OPT)
    = 0.
\end{align*}
Now assume that Eq.~\eqref{eq:submod_partial_approx} holds for every $0 \leq i \leq j < r$; we show that it also holds for $x_{j+1}$. To this end the algorithm needs to (i) select the $x_j \in M$ with $|x_j|_1 = j$ and (ii) add the element with the largest possible increase in $f$ in a single mutation. Let $\delta_{j+1}$ be the increase in fitness we obtain by this event. Then we get
\begin{align*}
    f(\OPT) \leq f(x_j \cup \OPT) \leq f(x_j) + r \cdot \delta_{j+1}
\end{align*}
where the first inequality follows from monotonicity of $f$ and the second by submodularity. Rearranging terms we obtain
\begin{align*}
    \delta_{j+1} \geq \frac{1}{r} \cdot \left(f(\OPT) - f(x_j)\right).
\end{align*}
This leads to
\begin{align*}
f(x_{j+1})
& \geq f(x_j) + \frac{1}{r} \cdot \left(f(\OPT) - f(x_j)\right) \\
& = \frac{f(\OPT)}{r} + f(x_j) \cdot \left(1 - \frac{1}{r}\right) \\
& \geq f(\OPT) \cdot \left(\frac{1}{r}  + \left(1 - \left(1 - \frac{1}{r}\right)^j\right) \cdot \left(1 - \frac{1}{r}\right)\right) \\
& = f(\OPT) \cdot \left(1 - \left(1 - \frac{1}{r}\right)^{j+1}\right).
\end{align*}
The latter for $j = r$ simplifies to
\begin{align*}
    f(\OPT) \cdot \left(1 - \left(1 - \frac{1}{r}\right)^{j+1}\right)
    \geq (1 - 1/e) \cdot f(\OPT)
\end{align*}
proving the claimed approximation ratio.
Note that the required selection plus mutation event happens at least with probability $\frac{1}{n+1} \cdot \frac{1}{en}$ where the first factor is the probability to select the respective parent and the second factor is a lower bound on the probability to greedily add the element with largest increase in fitness. Summing over all $r$ iterations we obtain an expected runtime of $\sum_{j=0}^{r-1} (n+1)ne = O(n^2r)$.
Adding the initial time to hit $0^n$ yields $O(n^2\log n) + O(n^2r) = O(n^2(\log(n) + r))$ and completes the proof.
\end{proof}

We see that \QD achieves the same approximation ratio as GSEMO and it does so in the same expected runtime. This comes as little surprise, as the algorithms work similarly (population with at most $(n+1)$ non-dominated individuals versus a map with at most $(n+1)$ individuals) as pointed out before.
However, \QD is a much simpler algorithm. The fitness function is the set function itself and no bi-objective formulation is required (that role is played by the feature space). As a consequence, the algorithm does not need to perform dominance checks in every iteration.
We are confident that many other results from~\cite{FN_2015_submodular_maximisation} also hold for~\QD.

\section{Minimum Spanning Tree Problem}
\label{sec:runtime_mst}



We now turn the focus to a problem-tailored feature space and a well-known combinatorial optimisation problem.
Consider a connected graph $G=(V,E)$ with $n=|V|$ nodes and $m = |E|$ edges and a weight function $w : E \to \mathbb{N}^{+}$ that maps each edge to a positive integer weight. Every connected acyclic sub-graph of $G$ is called a \emph{spanning tree}~(ST); let $\mathcal{T}$ be the set of all spanning trees of $G$. A spanning tree $T^{*}$ with $T^{*} = \argmin_{T \in \mathcal{T}} \sum_{e \in T} w(e)$ is a \emph{minimum spanning tree}~(MST) of $G$.
To simplify things we assume that all edge weights are pairwise different. In consequence, there is a unique MST.

We encode solution candidates as bit-strings $x$ from the search space $S=\{0,1\}^m$ where $x_i = 1$ means that the $i$th edge is in the solution and $x_i=0$ otherwise given an arbitrary, but fixed order of the edges $e_1, \ldots, e_m$. The fitness of solution candidate $x$ is -- overloading the function $w$ -- defined as
\begin{align*}
    w(x) := \sum_{i=1}^{m} x_i \cdot w(e_i).
\end{align*}
Let $cc(x)$ be the number of connected components of $x \in S$.

The MST problem was studied for Randomised Local Search~\cite{Neumann2006c}, ant colony optimisation~\cite{Neumann2010} and simulated annealing~\cite{Wegener2005,DoerrRW22}.

Neumann \& Wegener~\cite{Neumann2006c} used a bi-objective formulation of the MST problem where the fitness $f' : S \to \mathbb{R}^2$ of an individual $x \in S$ was defined by $w'(x) = (cc(x),w(x))$ which is to be minimised in both objectives simultaneously. Since $cc(x) \in \{1, \ldots, n\}$ the Pareto-front has size $n$ and thus the population size of SEMO and GSEMO is bounded from above by $n$. They showed that both SEMO and GSEMO, given an arbitrary initial search point, solve the bi-objective variant of the MST in expected time $O(nm \log(n w_{\max}))$ where $w_{\max} = \max_{e \in E} w(e)$ is the maximum edge weight. If weights are at most exponential in~$n$, ${w_{\max} \leq 2^{O(n)}}$, this yields an upper bound of $O(n^2m)$ since ${\log w_{\max} \le \log 2^{O(n)} = O(n)}$.

In the following we derive an upper bound for \QD with the feature space being defined as the number of CCs, $cc(x) \in \{1, \ldots, n\}$, of solution $x$. Thus, $L = n$ holds for the map size. For each number of CCs we store the best-so-far solution in the map. We refer to this as \emph{connected-components feature space}.

\begin{theorem}
\label{thm:qd1_mst_allzeroes_hitting_time}
\QD operating on the connected-components feature space locates $0^m$ (the empty edge set) in an expected time of $O(nm\log (nw_{\max}))$ where $w_{\max}$ is the maximum edge weight.
\end{theorem}

\begin{proof}
We follow well-established arguments for the analysis of the \EA and GSEMO on MST.
Let $x_t$ denote a solution with minimum weight amongst all solutions stored by \QD at time~$t$. Note that $w(x_t)$ is non-increasing over time.
If there are $k$ edges selected in $x_t$, there are $k$ distinct mutations flipping only a single selected edge, and not flipping any unselected edges. As all these mutations, when applied at the same time, decrease the weight from $w(x_t)$ to 0, the expected decrease in weight, when applying one such mutation chosen uniformly at random, is $w(x_t)/k$. The event we condition on has probability $k \cdot 1/m \cdot (1-1/m)^{m-1} \ge k/(em)$. In addition, the probability of selecting $x_t$ as parent is at least $1/n$ as there are at most $n$ populated cells in \QD's map. Together,
\[
    \E{w(x_{t+1}) \mid x_t} \le w(x_t) \cdot \left(1 - \frac{1}{enm}\right).
\]
The weight of the initial solution is $w_0 \le n w_{\max}$. Applying the multiplicative drift theorem yields an upper bound of $O(nm \log(n w_{\max}))$.
\end{proof}

The time bound matches the one for the expected time until the population of GSEMO contains $0^m$ (see \cite[Theorem~1]{NW2005_mst_multi}).

We remark that the preliminary version of this work published at GECCO 2023~\cite{Bossek2023} claims an upper bound of $O(nm \log n)$ in its Theorem~6.1. This result matches the above statement when all weights are polynomial in $n$. Regrettably, the proof in~\cite{Bossek2023} is flawed as mutations adding edges while also removing edges may increase the total number of edges. Hence, the argument used in~\cite{Bossek2023} that \QD may remove all edges by mutations flipping one bit is not sufficient. The authors thank Frank Neumann for pointing this out. To our knowledge the upper bound of $O(nm \log(n w_{\max}))$ is the best known upper bound for the expected time of the (1+1)~EA and GSEMO for finding an empty selection of edges in the MST problem~\cite{Neumann2005}, along with a cruder bound of $O(nm^2 \log n)$ due to Reichel and Skutella~\cite{Reichel2009} that removes the dependence on weights. We conjecture that a time bound of $O(nm \log n)$ holds for all weights, but we currently do not have a proof and therefore have to leave this for future work.

\begin{theorem}
\label{thm:qd1_mst}
\QD operating on the connected-components feature space and starting with $0^m \in M$, finds a minimum spanning tree in expected time $O(n^2m)$.
\end{theorem}

\begin{proof}
For $i = 1, \ldots, n$ let $x_i^{*}$ be the optimal spanning forest with $i$ connected components, i.e., \begin{align*}
    cc(x_i^{*}) = i \quad \text{ and } \quad w(x_i^{*}) = \min_{x \in S, cc(x)=i} w(x).
\end{align*}
It holds that $x_n^{*}=0^m$ since $cc(0^m)=n$ and $0^m$ is the empty edge set; all edge weights are positive and thus adding edges can only increase the fitness. Clearly, $x_1^{*}$ is a minimum spanning tree of the input graph. Given $x_i^{*}, i \in \{2, \ldots, n\}$ we can reach $x_{i-1}^{*}$ by injecting an edge of minimal weight into $x_i^{*}$ which does reduce the number of CCs by one or, put differently, does not close a cycle. We call such a step a \emph{Kruskal-step} as it is exactly the way Kruskal's well-known MST-algorithm builds an MST (see, e.g., \cite{Cormen2009_algorithms}). Now, starting with $x_n^{*}$ we can create the sequence $x_{n-1}^{*}, x_{n-2}^{*}, \ldots, x_1^{*}$ by $n-1$ consecutive Kruskal-steps. Having $x_i^{*} \in M$ the probability for a Kruskal-step is at least $1/(enm)$ as we need to select the right cell and flip exactly one bit.
In consequence, a MST is found after at most $(n-1)enm = O(n^2m)$ steps in expectation.
\end{proof}

Combining the results of Theorem~\ref{thm:qd1_mst_allzeroes_hitting_time} and Theorem~\ref{thm:qd1_mst} yields the following result. For edge weights bounded from above by $2^{O(n)}$ this matches the result by Neumann \& Wegener for (G)SEMO.

\begin{corollary}
\label{cor:qd1_mst}
On every connected graph $G=(V,E)$ with $n$ nodes, $m$ edges and $w : E \to \mathbb{N}^{+}$ with $w(e) \le 2^{O(n)}$ for all $e \in E$, \QD operating on the connected-components feature space and starting with an arbitrary initial solution finds a minimum spanning tree of $G$ in expected time $O(n^2m)$.
\end{corollary}

\section{Conclusion}
\label{sec:conclusion}

We performed runtime analyses of a simple quality diversity algorithm termed \QD in the context of pseudo-Boolean optimisation and combinatorial optimisation. 
For the number-of-ones feature space with granularity parameter~$k$ we showed an upper bound on the expected time until all $L$ cells are covered, for different values of $k$. We gave upper bounds on the expected optimisation time on  \ONEMAX, functions of unitation and monotone functions, for which the feature space is aligned favourably with the structure of the problem. We  showed that \QD finds a $(1-1/e)$-approximation for the task of monotone submodular function maximisation, given a single uniform cardinality constraint. Finally, we considered a different feature space spanned by the number of connected components of a connected edge-weighted graph and showed that QD finds a minimum spanning tree in expected time $O(n^2m)$ by imitating Kruskal's well-known MST algorithm, for integer edge weights up to $2^{O(n)}$.
Interestingly, we find that the working principles of \QD{} are similar to those of GSEMO while its implementation is easier and there is no need to use the concept of Pareto-dominance.

We see this paper as a starting point for a deep dive into the theoretical study of QD-algorithms. One important question for future work is to investigate the alignment between the feature space and the problem structure in more depth, to see how the efficiency of QD is affected if there is a mismatch between the two.
Straight-forward extensions of QD might consider more combinatorial optimisation problems, e.g., graph colouring problems, the study of different (multi-)dimensional feature spaces, and supplementary experiments. We also see different directions for methodological/algorithmic extensions, e.g., keeping multiple solutions per cell or combining classic population-based EAs with quality diversity approaches. Another promising research direction is identifying general relationships between QD and GSEMO.

\balance
\bibliographystyle{ACM-Reference-Format}
\bibliography{bib}



\end{document}

%% file: includes/commands.tex
\newcommand{\EA}{$(\mu + 1)$-EA\xspace}
\newcommand{\argmax}{\text{arg\,max}}
\newcommand{\argmin}{\text{arg\,min}}
\newcommand{\OPT}{\text{OPT}}
\newcommand{\ones}[1]{|#1|_1}
\newcommand{\Prob}[1]{\mathrm{Pr}\left(#1\right)}
\newcommand{\prob}[1]{\Prob{#1}}
\newcommand{\E}[1]{\text{E}\left(#1\right)}

\newcommand{\ONEMAX}{\textsc{OneMax}\xspace}

\newcommand{\QD}{QD\xspace}
\newcommand{\QDk}{QD({\ensuremath{k}})\xspace}
\newcommand{\noo}[1]{\ensuremath{#1}\nobreakdash-NoO}

\newcommand{\mplEA}[2]{$(#1+#2)$-EA\xspace} 
\newcommand{\opoEA}[1]{\mplEA{1}{1}} 

\newcommand{\Wlog}{W.\,l.\,o.\,g.\xspace}

%% file: includes/commenting_macros.tex
\newcommand{\hide}[1]{}

\usepackage{xcolor}

\definecolor{jakobcolor}{RGB}{138,43,226}
\definecolor{dirkcolor}{RGB}{60,179,113}

\definecolor{todocolor}{rgb}{0.9,0.1,0.1}
\definecolor{changedcolor}{rgb}{0.42,0.27,0.57}
\definecolor{addedcolor}{rgb}{0.867,0.176,0.361}

\newcommand{\nbc}[3]{
		{\colorbox{#3}{\bfseries\sffamily\scriptsize\textcolor{white}{#1}}}
		{\textcolor{#3}{\sf\small$\blacktriangleright$\textit{#2}$\blacktriangleleft$}}
}
	
\newcommand{\jakob}[1]{\nbc{Jakob}{#1}{jakobcolor}}
\newcommand{\dirk}[1]{\nbc{Dirk}{#1}{dirkcolor}}

\newcommand{\todo}[1]{\nbc{TODO}{#1}{todocolor}}
\newcommand{\changed}[1]{\nbc{CHANGED}{#1}{changedcolor}}
\newcommand{\added}[1]{\nbc{ADDED}{#1}{addedcolor}}

\newcommand{\redacted}[1]{\emph{[anonymized for review]}}

 \renewcommand{\jakob}[1]{}
 \renewcommand{\dirk}[1]{}

\renewcommand{\changed}[1]{#1}
\renewcommand{\added}[1]{#1}
\renewcommand{\todo}[1]{}
